\documentclass[11pt,twoside]{article}
\usepackage{geometry}
\usepackage{fancyhdr}
\fancypagestyle{firststyle}
{
   \fancyhf{}
   \fancyfoot{}
}

\usepackage{amsmath,amsthm,amssymb}
\usepackage{graphicx}

\usepackage{hyperref}

\usepackage{mathtools}

\newcommand{\R}{\mathbb{R}}
\newcommand{\Order}{\mathcal{O}}
\DeclareMathOperator{\sign}{sign}

\usepackage[english]{babel}

\usepackage{rotating}
\usepackage{paralist}

\usepackage{wrapfig}

\usepackage{subcaption}
\usepackage[linesnumbered,ruled,noend]{algorithm2e}
\newcommand\mycommfont[1]{\footnotesize{1}}

\usepackage{url}

\newtheorem{lemma}{Lemma}

\begin{document}

\title{Speeding Up Budgeted Stochastic Gradient Descent SVM Training with Precomputed Golden Section Search}
\author{Tobias Glasmachers and Sahar Qaadan\\Institut f{\"u}r Neuroinformatik, Ruhr-Universit\"at Bochum, Germany\\\texttt{tobias.glasmachers@ini.rub.de, sahar.qaadan@ini.rub.de}}
\date{}

\maketitle

\begin{abstract}
Limiting the model size of a kernel support vector machine to a
pre-defined budget is a well-established technique that allows to
scale SVM learning and prediction to large-scale data. Its core
addition to simple stochastic gradient training is budget maintenance
through merging of support vectors. This requires solving an inner
optimization problem with an iterative method many times per gradient
step. In this paper we replace the iterative procedure with a fast
lookup. We manage to reduce the merging time by up to $65\%$ and the
total training time by $44\%$ without any loss of accuracy.
\end{abstract}

\section{Introduction}

The Support Vector Machine (SVM; \cite{cortes1995support}) is a 
widespread standard machine learning method, in particular for binary
classification problems.
Being a kernel method, it employs a linear algorithm in an implicitly
defined kernel-induced feature space \cite{Vapnik:95}. 
SVMs yield high predictive accuracy in many applications \cite{Noble:2004,Lewis:2006a,Cui:2007,Lin:2014,Yu:2016}.
They are supported by strong learning theoretical guarantees \cite{Joachims:98,Bottou:2007,Mohri:2012,Hare:2016}.

When facing large-scale learning, the applicability of support vector
machines (and many other learning machines) is limited by their
computational demands. Given $n$ training points, training an SVM with
standard dual solvers takes quadratic to cubic time in $n$
\cite{Bottou:2007}. Steinwart \cite{steinwart2003sparseness} established
that the 
number of support vectors is linear in $n$, and so is the storage
complexity of the model as well as the time complexity of each of its
predictions. This quickly becomes prohibitive for large $n$, e.g., when
learning from millions of data points.

Due to the prominence of the problem, a large number of solutions was
developed. Parallelization can help \cite{Zanni:2006,Zhu:2009}, but it
does not reduce the complexity of the training problem.
One promising route is to solve the SVM problem only locally, usually
involving some type of clustering \cite{zhang2006svm,ladicky2011locally}
or with a hierarchical divide-and-conquer strategy \cite{graf:2005,hsieh2014divide}.
An alternative approach is to leverage the progress in the domain of
linear SVM solvers \cite{Joachims:99,Zhang:2004,Hsieh:2008}, which scale well to large data sets. To this end, kernel-induced feature
representations are approximated by low-rank approaches
\cite{fine2001efficient,Rahimi:2007,zhang2012scaling,Yadong:2014}, either a-priory using random Fourier features, or in a data-dependent
way using Nystr{\"o}m sampling.

Budget methods, introducing an a-priori limit $B \ll n$ on the number of
support vectors \cite{Nguyen:2005,Wang:2012}, go one step
further by letting the optimizer adapt the feature space approximation
during operation to its needs, which promises a comparatively low
approximation error. The usual strategy is to merge support vectors at
need, which effectively enables the solver to move support vectors
around in input space. Merging decisions greedily minimize the
approximation error.

In this paper we propose an effective computational improvement of this
scheme. Finding the best merge partners, i.e., support vectors that
induce the lowest approximation error when merged, is a rather costly
operation. Usually, $\Order(B)$ candidate pairs of vectors are
considered, and for each pair an optimization problem is solved with an
iterative strategy. By modelling the low-dimensional space of (solutions
of the) optimization problems explicitly, we can remove the iterative
process entirely, and replace it with a simple and fast lookup.

Our results show that merging-based budget maintenance can account for
more than half of the total training time. Therefore reducing the
merging time is a promising approach to speeding up training. The
speed-up can be significant; on our largest data set we reduce the
merging time by $65\%$, which corresponds to a reduction of the total
training time by $44\%$.
At the same time, our lookup method is at least as accurate as the
original iterative procedure, resulting in nearly identical merging
decisions and no loss of prediction accuracy.

The reminder of this paper is organized as follows. In the next section
we introduce SVMs and stochastic gradient training on a budget. Then we
analyze the computational bottleneck of the solver and develop a lookup
smoothed with bilinear interpolation as a remedy.
In section~\ref{section:experiments} we benchmark the new algorithm
against ``standard'' BSGD, and we investigate the influence of the
algorithmic simplification on different budget sizes. Our results
demonstrate systematic improvements in training time at no cost in terms
of solution quality.

\section{Support Vector Machine Training}

In this section we introduce the necessary background: SVMs for binary
classification, and training with stochastic gradient descent (SGD) on a
budget, i.e., with a-priori limited number of support vectors.

\paragraph{Support Vector Machines}

An SVM classifier is a supervised machine learning algorithm.
In its simplest form it linearly separates two classes with a large
margin. When applying a kernel function $k : X \times X \to \R$ over the
input space $X$, the separation happens in a reproducing kernel Hilbert
space (RKHS). For labeled data
$\big((x_1, y_1), \dots, (x_n, y_n)\big) \in (X \times \{-1, +1\})^n$,
the prediction on $x \in X$ is computed as
\begin{align*}
	\sign \Big( \big\langle w, \phi(x) \big\rangle + b \Big) &= \sign \left( \sum_{j=1}^n \alpha_j k(x_j, x) + b \right)
\end{align*}
with $w = \sum_{j=1}^n \alpha_j \phi(x_j)$,
where $\phi(x)$ is an only implicitly defined feature map (due to
Mercer's theorem, see also \cite{Vapnik:95}) corresponding to the
kernel function fulfilling $k(x, x') = \langle \phi(x), \phi(x') \rangle$.
Training points $x_j$ with non-zero coefficients $\alpha_j \not= 0$ are
called support vectors; the summation in the predictor can apparently be
restricted to this subset. 
The SVM model is obtained by minimizing the following (primal) objective
function:
\begin{align}
	P(w, b) = \frac{\lambda}{2} \|w\|^2 + \frac{1}{n} \sum_{i=1}^n L \Big( y_i, \big\langle w, \phi(x_i) \big\rangle + b \Big).
	\label{eq:primal}
\end{align}
Here, $\lambda > 0$ is a user-defined regularization parameter and
$L(y, \mu) = \max\{0, 1 - y \cdot \mu\}$ denotes the hinge loss, which
is a prototypical large margin loss, aiming to separate the classes with
a functional margin $y \cdot \mu$ of at least one. The incorporation of
other loss functions allows to generalize SVMs to other tasks like
multi-class classification, regression, and ranking.

\paragraph{Primal Training}

Problem~\eqref{eq:primal} is a convex optimization problem without
constraints. It has an equivalent dual representation as a quadratic
program (QP), which is solved by several state-of-the-art ``exact''
solvers like LIBSVM \cite{Chang:2011} and thunder-SVM \cite{thundersvm}.
The main challenge is the high dimensionality of the problem, which
coincides with the training set size~$n$ and can hence easily grow into
the millions.

A simple method is to solve problem~\eqref{eq:primal} directly with
stochastic gradient descent (SGD), similar to neural network training.
When presenting one training point at a time, as done in Pegasos
\cite{shalev2011pegasos}, the objective function $P(w, b)$ is
approximated by the unbiased estimate
\begin{align*}
	P_i(w, b) = \frac{\lambda}{2} \|w\|^2 + L \Big( y_i, \big\langle w, \phi(x_i) \big\rangle + b \Big),
\end{align*}
where the index $i \in \{1, \dots, n\}$ follows a uniform distribution.
The stochastic gradient $\nabla P_i(w, b)$ is an unbiased estimate of
the ``batch'' gradient $\nabla P(w, b)$ but faster to compute by a
factor of $n$, since it involves only a single training point.
Starting from $(w, b) = (0, 0)$, SGD updates the weights according to
\begin{align*}
	(w, b) \leftarrow (w, b) - \eta_t \cdot \nabla P_{i_t}(w, b),
\end{align*}
where $t$ is the iteration counter. With a learning rate $\eta_t \in
\Theta(1/t)$ it is guaranteed to converge to the optimum of the convex
training problem \cite{Bottou:2010}.

With a sparse representation
$w = \sum_{(\alpha, \tilde{x}) \in M} \alpha \cdot \phi(\tilde{x})$
the SGD update decomposes into the following algorithmic steps. We scale
down all coefficients $\alpha$ uniformly by the factor
$1 - \lambda \cdot \eta_t$. If the margin
$y_i (\langle w, \phi(x_i)\rangle + b)$ happens to be less than one,
then we add a new point $\tilde{x} = x_i$ with coefficient
$\alpha = \eta_t \cdot y_i$ to the model~$M$.
With a dense representation holding one coefficient $\alpha_i$ per data
point $(x_i, y_i)$ we would add the above value to $\alpha_i$. The most
costly step is the computation of $\langle w, \phi(x_i) \rangle$, which
is linear in the number of support vectors (SVs), and hence generally
linear in~$n$ \cite{steinwart2003sparseness}.

\paragraph{SVM Training on a Budget}

Budgeted Stochastic Gradient Descent (BSGD) breaks the unlimited growth
in model size and update time for large data streams by bounding the
number of support vectors during training. The upper bound $B \ll n$ is
the budget size. Per SGD step the algorithm can add at most one new
support vector; this happens exactly if $(x_i, y_i)$ does not meet the
target margin of one (and $\alpha_i$ changes from zero to a non-zero
value). After $B+1$ such steps, the budget constraint is violated and a
dedicated budget maintenance algorithm is triggered to reduce the number
of support vectors to at most~$B$. The goal of budget maintenance is to
fulfill the budget constraint with the smallest possible change of the
model, measured by $\|\Delta\|^2 = \|w' - w\|^2$, where $w$ is the
weight vector before and $w'$ is the weight vector after budget
maintenance. $\Delta = w' - w$ is referred to as the weight degradation.

Budget maintenance strategies are investigated in detail in
\cite{Wang:2012}. It turns out that \emph{merging} of two support
vectors into a single new point is superior to alternatives like removal
of a point and projection of the solution onto the remaining support
vectors.
Merging was first proposed in \cite{Nguyen:2005} as a way to efficiently
reduce the complexity of an already trained SVM. With merging, the
complexity of budget maintenance is governed by the search for suitable
merge partners, which is $\Order(B^2)$ for all pairs, while it is common
to apply the $\Order(B)$ heuristic resulting from fixing the point with
smallest coefficient $\alpha_i$ as a first partner.

When merging two support vectors $x_i$ and $x_j$, we aim to approximate
$\alpha_i \cdot \phi(x_i) + \alpha_j \cdot \phi(x_j)$ with a new term
$\alpha_z \cdot \phi(z)$ involving only a single point $z$.
Since the kernel-induced feature map is usually not surjective, the
pre-image of $\alpha_i \phi(x_i) + \alpha_j \phi(x_j)$ under $\phi$ is
empty \cite{sch-et-al:iskm,Burges96simplifiedsupport} and no exact match $z$ exists. 
Therefore the weight degradation
$\Delta = \alpha_i \phi(x_i) + \alpha_j \phi(x_j) - \alpha_z \phi(z)$
is non-zero. For the Gaussian kernel
$k(x, x') = \exp(-\gamma \|x - x'\|^2)$, due to its symmetries, the
point $z$ minimizing $\|\Delta\|^2$ lies on the line connecting $x_i$
and $x_j$ and is hence of the form $z = h x_i + (1 - h) x_j$. For
$y_i = y_j$ we obtain a convex combination $0 < h < 1$, otherwise we
have $h < 0$ or $h > 1$. In this paper we merge only vectors of equal
label. For each choice of $z$, the optimal value of $\alpha_z$ is
obtained in closed form:
$\alpha_z = \alpha_i k(x_i, z) + \alpha_j k(x_j, z)$.
This turns minimization of
$\|\Delta\|^2 = \alpha_i^2 + \alpha_j^2 - \alpha_z^2 + 2 k(x_i, x_j)$
into a one-dimensional non-linear optimization problem, which is solved
in \cite{Wang:2012} with golden section line search. The calculations
are further simplified by the relations
$k(x_i, z) = k(x_i, x_j)^{(1-h)^2}$ and $k(x_j, z) = k(x_i, x_j)^{h^2}$,
which save costly kernel functions evaluations.

Budget maintenance in BSGD usually works in the following sequence of
steps, see algorithm \ref{algo:budget}: First, $x_i$ is
fixed to the support vector with minimal coefficient $|\alpha_i|$. Then
the best merge partner $x_j$ is determined by testing $B$ pairs
$(x_i, x_j)$, $j \in \{1, \dots, B+1\} \setminus \{i\}$. Golden section
search is run for each of these steps to determine $h$ to fixed precision
$\varepsilon = 0.01$. The weight degradation is computed using the
shortcuts mentioned above. Finally, the candidate with minimal weight
degradation is selected and the vectors are merged. Hence, although a
single golden search search is fast, the need to run it many times per
SGD iteration turns it into a rather costly operation.
	 \begin{algorithm}
	 
		\textbf{Input/Output:} model $M$

		{$(\alpha_{\min}, \tilde{x}_{\min}) \leftarrow \arg \min \big\{ |\alpha| \,\big|\, (\alpha, \tilde{x}) \in M \big\}$}

		{$W\!D^* \leftarrow \infty$}

		\For{ $(\alpha, \tilde{x}) \in M \setminus \{(\alpha_{\min}, \tilde{x}_{\min})\}$ }
		{
			{$m \leftarrow \alpha / (\alpha + \alpha_{\min})$}

			{$\kappa \leftarrow k(\tilde{x}, \tilde{x}_{\min})$}

			{$h \leftarrow \arg\max \big\{ m \kappa^{(1 - h')^2} + (1-m)\kappa^{h'^2} \big| h' \in [0, 1] \big\}$ \label{line:argmax}}

			{$\alpha_z \leftarrow \alpha_{\min} \cdot \kappa^{(1 - h)^2} + \alpha \cdot \kappa^{h^2}$}

			{$W\!D \leftarrow \alpha_{\min}^2 + \alpha^2 - \alpha_z^2 + 2 \cdot \alpha_{\min} \cdot \alpha \cdot \kappa$}

			\If{$(W\!D<W\!D^*)$}
			{
				{$W\!D^* \leftarrow W\!D$}

				{$(\alpha^*, \tilde{x}^*, h^*, \kappa^*) \leftarrow (\alpha, \tilde{x}, h, \kappa)$}
			}
		}

		{$z \leftarrow h^* \cdot \tilde{x}_{\min} + (1 - h^*) \cdot \tilde{x}^*$}

		{$\alpha_z \leftarrow \alpha_{\min} \cdot (\kappa^*)^{(1-h^*)^2} + \alpha^* \cdot (\kappa^*)^{(h^*)^2}$}

		{$M \leftarrow M \setminus \{(\alpha_{\min}, \tilde{x}_{\min}), (\alpha^*, \tilde{x}^*)\} \cup \{(\alpha_z, z)\}$}

		\caption{\label{algo:budget}
		Procedure Budget Maintenance for a sparse model $M$
		}
	\end{algorithm}

A theoretical analysis of BSGD is provided by \cite{Wang:2012}. Their
Theorem~1 establishes a bound on the error induced by the budget,
ensuring that asymptotically the error is governed only by the
(unavoidable) weight degradation.

\section{Precomputing the Merging Problem}

\begin{wrapfigure}{r}{0.48\textwidth}
	\begin{center}
		\includegraphics[width=0.46\textwidth]{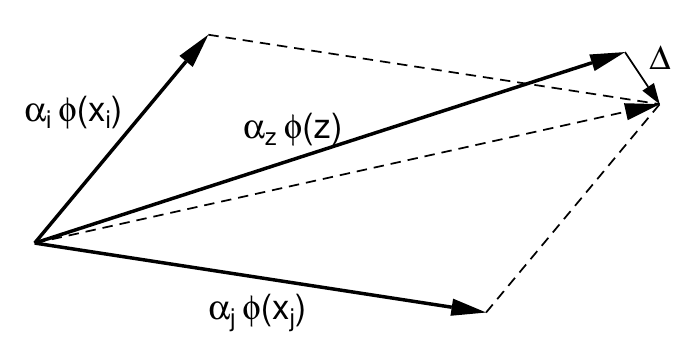}
	\end{center}
	\caption{The merging problem.
		\label{figure:merging}
	}
\end{wrapfigure}

The merging problem for given support vectors $x_i$ and $x_j$ with
coefficients $\alpha_i$ and $\alpha_j$ is illustrated in
figure~\ref{figure:merging}. Our central observation is that the
geometry depends only on the (cosine of the) angle between
$\alpha_i \phi(x_i)$ and $\alpha_j \phi(x_j)$, and on the relative
lengths of the two vectors. These two quantities are captured by the
parameters
\begin{compactitem}
\item
	relative length $m = \alpha_i / (\alpha_i + \alpha_j)$
\item
	cosine of the angle $\kappa = k(x_i, x_j)$,
\end{compactitem}
both of which take values in the unit interval. The optimal merging
coefficient $h$ is a function of $m$ and $\kappa$, and so is the
resulting weight degradation $W\!D = \|\Delta\|^2$. Therefore we can
express $h$ and $W\!D$ as functions of $m$ and $\kappa$, denoted as
$h(m, \kappa)$ and $W\!D(m, \kappa)$ in the following. The functions can
be evaluated to any given target precision by running the golden section
search. Their graphs are plotted in figures~\ref{figure:functionsb} and~\ref{figure:functionsc}.

If the functions $h$ or $W\!D$ can be approximated efficiently then there
is no need to run a potentially costly iterative procedure like golden
section search. This is our core technique for speeding up the BSGD
method.

The functions blend between different budget maintenance strategies.
While for $\kappa \gg 0$ and for $m \approx 1/2$ it is beneficial to
merge the two support vectors, resulting in $h \in (0, 1)$, this is not
the case for $\kappa \ll 1$ and $m \approx 0$ or $m \approx 1$,
resulting in $h \approx 0$ or $h \approx 1$, which is equivalent to
removal of the support vector with smaller coefficient. This means that
in order to obtain a close fit that works well in both regimes we may
need a quite flexible function class like a kernel method or a neural
network, while a simple polynomial function can give poor fits, with
large errors close to the boundaries.

A much simpler and computationally very cheap approach is to pre-compute
the function on a grid covering the domain $[0, 1] \times [0, 1]$. The
values need to be pre-computed only once, and here we can afford to
apply golden section search with high precision; we use
$\varepsilon = 10^{-10}$. Then, given two merge candidates, we can look
up an approximate solution by rounding $m$ and $\kappa$ to the nearest
grid point. The approximation quality can be improved significantly
through bilinear interpolation. On modern PC hardware we can easily
afford a large grid with millions of points, however, this is not even
necessary to obtain excellent results. In our experiments we use a grid
of size $400 \times 400$.

Bilinear interpolation is fast, and moreover it is easy to implement.
When looking up $h(m, \kappa)$ this way, we obtain a plug-in replacement
for golden section search in BSGD. However, we can equally well look up
$W\!D(m, \kappa)$ instead to save additional computation steps.
Another benefit of $W\!D$ over $h$ is regularity, see
figures \ref{figure:functionsb} and \ref{figure:functionsc}
and the following lemma.

\begin{figure}
\begin{subfigure}{0.45\textwidth}
\includegraphics[width=\linewidth]{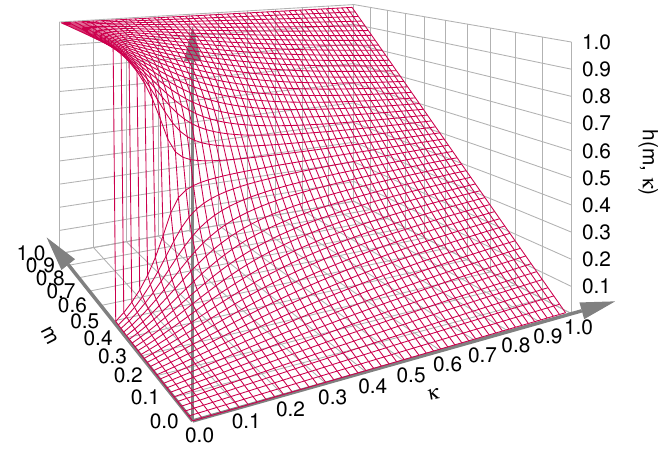}
\caption{} \label{figure:functionsb}
\end{subfigure}
\hspace*{\fill} 
\begin{subfigure}{0.45\textwidth}
\includegraphics[width=\linewidth]{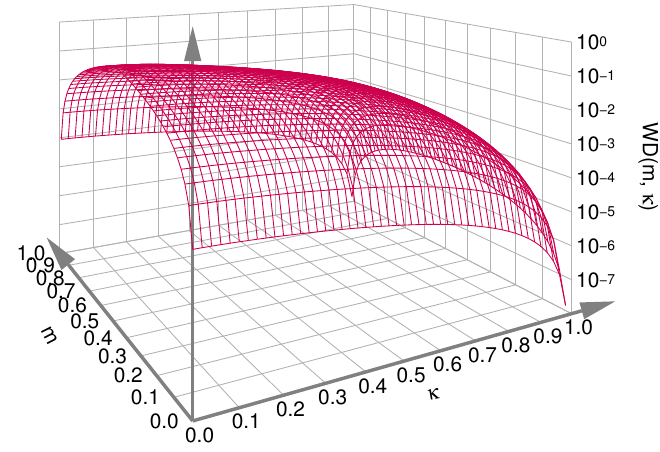}
\caption{}\label{figure:functionsc}
\end{subfigure}
\caption{Graphs of the functions $h(m, \kappa)$ (\ref{figure:functionsb})
	and $W\!D(m, \kappa)$ (\ref{figure:functionsc}).
	The latter uses a log scale on the value axis.
	\label{fig:1}
	}
\end{figure}

\begin{lemma}
The functions $h$ and $W\!D$ are smooth for $\kappa > e^{-2}$.
The function $h$ is continuous outside the set
$Z = \{1/2\} \times [0, e^{-2}] \subset [0, 1]^2$ and discontinuous
on $Z$. The function $W\!D$ is everywhere continuous.
\end{lemma}
\begin{proof}
The function $s_{m,\kappa}(h') = m \kappa^{(1-h')^2} + (1-m)
\kappa^{h'^2}$ used in line~\ref{line:argmax} of
algorithm~\ref{algo:budget} inside the $\arg\max$ expression is a
weighted sum of two Gaussian kernels. Depending on the parameters $m$
and $\kappa$, it can have one or two modes. It has two modes for
parameters in $Z$, as can be seen from an elementary calculation
yielding
$s''_{1/2,\kappa}(1/2) > 0 \Leftrightarrow \kappa < e^{-2}$. Due to
symmetry, the dominant mode switches at $m=1/2$. The inverse function
theorem applied to branches of $s_{m, \kappa}$ implies that
$h(m, \kappa) = \arg\max_{h'}\{s_{m,\kappa}(h')\}$ and
$W\!D(m, \kappa) = (\alpha_i + \alpha_j) \cdot \big(m^2 + (1-m)^2 - [s_{m,\kappa}(h(m,\kappa))]^2 + 2m(1-m)\kappa\big)$
vary smoothly with their parameters as long as the same mode is active.
The maximum operation is continuous, and so is $W\!D$. For each $m$ there
is a critical value of $\kappa \leq e^{-2}$ where $s_{m,\kappa}$
switches from one to two modes. We collect these parameter
configurations in the set $N$. On $N$ (in contrast to $Z$), $h$ is
continuous. With the same argument as above, $h$ and $W\!D$ are smooth
outside $N \cup Z$.
\hfill$\qed$
\end{proof}

Bilinear interpolation is well justified if the function is continuous,
and differentiable within each grid cell. The above lemma ensures this
property for $\kappa > e^{-2}$, and it furthermore indicates that for
its continuity, interpolating $W\!D$ is preferable over interpolating
$h$. The regime $\kappa < e^{-2}$ corresponds to merging two points in
a distance of more than two ``standard deviations'' of the Gaussian
kernel. This is anyway undesirable, since it can result in a large
weight degradation. In fact, if $s_{m,\kappa}$ has two modes, then the
optimal merge is close to the removal of one of the points, which is
known to give poor results \cite{Wang:2012}.
\section{Experimental Evaluation}
\label{section:experiments}

In this section we evaluate our method empirically, with the aim to
investigate its properties more closely, and to demonstrate its
practical value. To this end, we'd like to answer the following
questions:

\begin{compactenum}
\item
	Which speed-up is achievable?
\item
	Do we pay for speed-ups with reduced test accuracy?
\item
	How do results depend on the budget size?
\item
	How much do merging decision differ from the original method?
\end{compactenum}

To answer these questions we compare our algorithm to ``standard'' BSGD
with merging based on golden section search. We have implemented both
algorithms in C++; the implementation is available from the first
author's homepage.%
\footnote{
\url{https://www.ini.rub.de/the_institute/people/tobias-glasmachers/\#software}
}
We train SVM models on the binary classification problems SUSY, SKIN,
IJCNN, ADULT, WEB, and PHISHING,
covering a range of different sizes. The regularization parameter
$C = \frac{1}{n \cdot \lambda}$ and the kernel parameter $\gamma$ were
tuned with grid search and cross-validation. The data sets are
summarized in table~\ref{table:dataSetsProperties}. SVMs were trained
with 20 passes through the data, except for the huge SUSY data, where we
used a single pass.

\begin{table}
\begin{center}

	\caption{
		\label{table:dataSetsProperties}
		Data sets used in this study, hyperparameter settings,
    	and test accuracy of the exact SVM model found by LIBSVM.
	}

	\resizebox{\columnwidth}{!}{
	 \begin{subtable}{0.55\linewidth}
	\begin{tabular}{lrrrrr}
	\hline
	data set & size & features & $C$ & $\gamma$ & accuracy \\
	\hline
	SUSY & 4,500,000 & 18 & $2^5$ & $2^{-7}$ & $79.79\%$ \\
	SKIN & 183,793 & 3 & $2^5$ & $2^{-7}$ & $99.96\%$ \\
	IJCNN & 49,990 & 22 & $2^5$ & $2^1$ & $98.77\%$ \\
	\hline
	\end{tabular}
	\end{subtable}%
	 \begin{subtable}{0.55\linewidth}
	\begin{tabular}{lrrrrr}
	\hline
	data set & size & features & $C$ & $\gamma$ & accuracy \\
	\hline
	ADULT & 32,561 & 123 & $2^5$ & $2^{-7}$ & $84.82\%$ \\
	WEB & 17,188 & 300 & $2^3$ & $2^{-5}$ & $98.81\%$ \\
	PHISHING & 8,315 & 68 & $2^3$ & $2^3$ & $97.55\%$ \\
	\hline
	\end{tabular}
	\end{subtable}%
	}
\end{center}
\end{table}

To answer the first question, we trained SVM models with BSGD,
comparing golden section search (GSS) with our new algorithms looking up
$h(m, \kappa)$ (Lookup-h) or $W\!D(m, \kappa)$ (Lookup-WD). For
reference, we also ran golden section search with precision
$\varepsilon = 10^{-10}$ (GSS-precise). We used two different budget
sizes for each problem.

All methods found SVM models with comparable accuracy as shown in
table~\ref{table:accuracy}; in fact, in most cases the systematic
differences are below one standard deviation of the variability between
different runs.%
\footnote{Note that with increasing number of passes (or epochs) the
standard deviation does not tend to zero since the training problem is
non-convex due to the budget constraint.}
In contrast, the time spent on budget maintenance differs significantly
between the methods. In figure~\ref{figure:TimeProfiles} we provide a
detailed breakdown of the merging time, obtained with a profiler.

Lookup-WD and Lookup-h are faster than GSS, which is (unsurprisingly)
faster than GSS-precise. The results are very systematic, see
table~\ref{table:training-time} and figure~\ref{figure:TimeProfiles}.
The greatest savings of about $44\%$ of the total training time are
observed for the rather large SUSY data set. Although the speed-up can
also be insignificant, like for the WEB data, lookup is never slower
than GSS. The actual saving depends on the cost of kernel computations
and on the fraction of SGD iterations in which merging occurs. The
latter quantity, which we refer to as the merging frequency, is provided
in table~\ref{table:training-time}.
We observe that the savings shown in figure~\ref{figure:TimeProfiles}
nicely correlate with the merging frequency.

The profiler results provide a more detailed understanding of the
differences: replacing GSS with Lookup-h significantly reduces the time
for computing $h(m, \kappa)$. Replacing Lookup-h with Lookup-WD removes
further steps in the calculation of $W\!D(m, \kappa)$, but practically
speaking the difference is hardly noticeable.

Overall, our method offers a systematic speed-up. The speed-up does not
come at any cost in terms of solution precision. This answers the first
two questions.

If the budget size is chosen so large that merging is never needed then
all tested methods coincide, however, this defeats the purpose of using
a budget in the first place. We find that the merging frequency is
nearly independent of the budget size as long as the budget is
significantly smaller than the number of support vectors of the full
kernel SVM model, and hence the fraction of runtime saved is independent
of the budget size. The results in figure~\ref{figure:TimeProfiles} are
in line with this expectation, answering the third question.

In the next experiment we have a closer look at the impact of
lookup-based merging decisions by investigating the behavior in single
iterations, as follows. During a run of BSGD we execute GSS and Lookup-WD
in parallel. We count the number of iterations in which the merging
decisions differ, and if so, we also record the difference between the
weight degradation values. The results are presented in
table~\ref{table:training-time}.
They show that the decisions of the two
methods agree most of the time, for some problems in more than $99\%$ of
all budget maintenance events.

Finally, we investigate the precision with which the weight degradation
is estimated by the different methods. While GSS can solve the problem
to arbitrary precision, the reference implementation determines
$h(m, \kappa)$ only to a rather loose precision of $\varepsilon = 0.01$ in
order to save computation time. In contrast, we ran GSS to high
precision $\varepsilon = 10^{-10}$ when precomputing the lookup table,
however, we may lose some precision due to bilinear interpolation. This
loss shrinks as the grid size grows, which comes at added storage cost,
but without any runtime cost.
We investigate the precision of GSS and Lookup-WD by comparing them to
GSS-precise, which is considered a reasonable approximation of the exact
minimum of $\|\Delta\|^2$. For both methods we record the factor by
which their squared weight degradations exceed the minimum, see
table~\ref{table:training-time}.
All factors are very close to one, hence none of the algorithms is
wasteful in terms of weight degradation, and indeed Lookup-WD with a
grid size of $400 \times 400$ is more precise on all 6 data sets. This
answers our last question.

\begin{table}
\begin{center}
	\caption{
		\label{table:accuracy}
		Test accuracy achieved by the different methods, averaged over 5 runs
		at different budget sizes.
	}
	\begin{tabular}{lccccc}
	\hline
	data set & budget & test accuracy& test accuracy & test accuracy & test accuracy \\
	         & size   & GSS-precise & GSS-standard           & Lookup-h      & Lookup-WD     \\
	\hline
	SUSY     & $100$ & $76.975	\pm1.372$&	$76.628\pm	2.030$&$	76.934\pm	1.426$&	$76.884\pm	1.261$ \\
	         
	         & $500$  &$76.989\pm	3.109$	&$75.583\pm	3.0558$	&$75.581\pm	2.558$&	$75.570\pm	3.925$\\
	         
	SKIN     & $100$  &$99.621\pm	0.711$&	$99.629	\pm0.852$&	$99.621\pm	0.201$	&$99.617\pm	0.877$\\
	
	         & $200$  & $99.868 \pm 0.033$ & $99.877 \pm 0.053$ & $99.855 \pm 0.054$ & $99.754 \pm 0.089$ \\

	IJCNN    & $100$  & $97.141 \pm 0.317$ & $96.807 \pm 0.344$ & $97.132 \pm 0.371$ & $97.130 \pm 0.363$ \\

	         & $500$  & $98.138 \pm 0.158$ & $98.055 \pm 0.334$ & $98.113 \pm 0.448$& $98.070 \pm 0.372$ \\

	ADULT    & $100$  & $84.234 \pm 0,883$ & $84.166 \pm 0.701$ & $84.164 \pm 0.988$ & $84.200 \pm 0.798$ \\

	         & $500$   & $84.280 \pm 0.800$ & $83.739 \pm 1.303$ & $83.836 \pm 1.157$ & $83.949 \pm 1.001$ \\
	         
	         WEB    & $100$  & $98.805 \pm 0,026$ & $98.793 \pm 0.027$ & $98.783 \pm 0.045$ & $98.793 \pm 0.039$ \\
   
	         & $500$   & $98.809 \pm 0,023$ & $98.781 \pm 0.047$ & $98.799 \pm 0.029$ & $98.807 \pm 0.016$ \\
	     
	PHISHING & $100$  & $96.554 \pm 0.158$ & $96.254 \pm 0.301$ & $96.539 \pm 0.242$ & $96.389 \pm 0.371$ \\

	         & $500$  & $97.555 \pm 0.187$ & $97.517 \pm 0.292$ & $97.518 \pm 0.280$ & $97.525 \pm 0.201$\\
	      
	\hline
	\end{tabular}
\end{center}
\end{table}

\begin{figure}[]%
    \centering
    \includegraphics[width=0.23\textwidth,height=0.16\textheight]{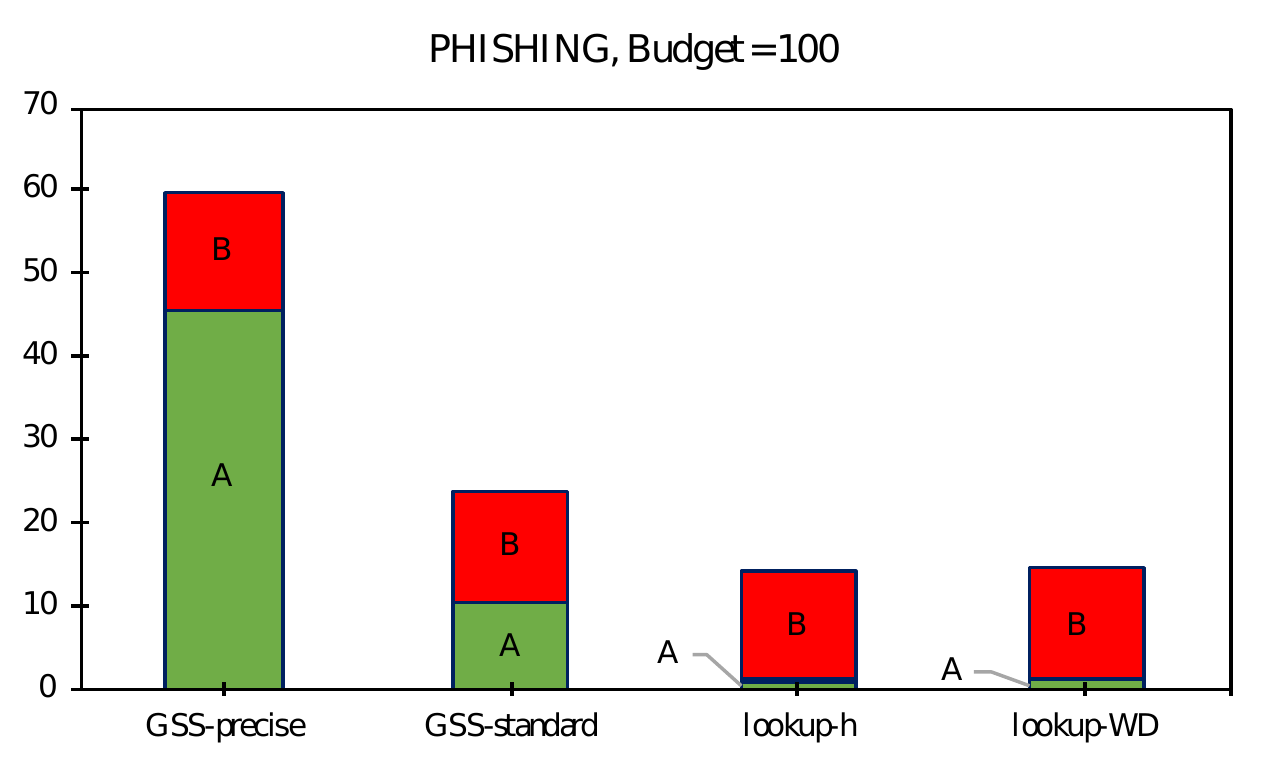}
    \includegraphics[width=0.23\textwidth,height=0.16\textheight]{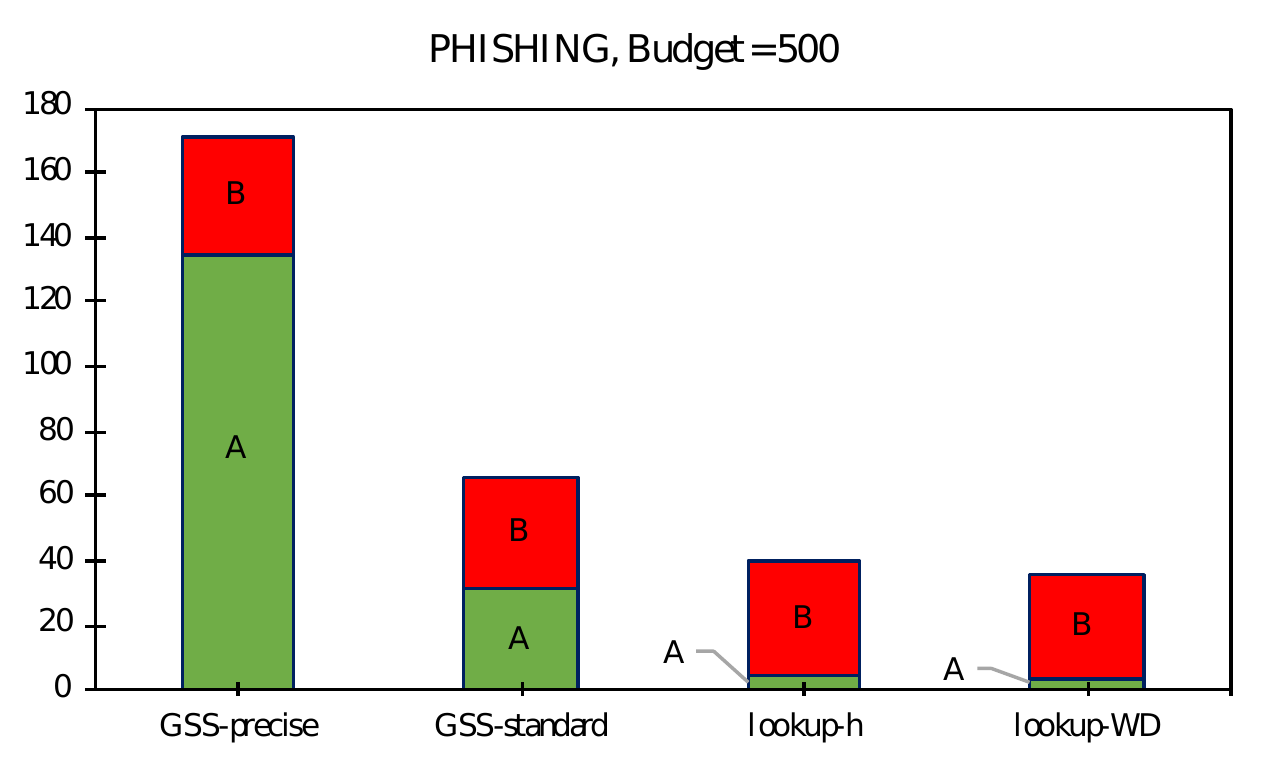}
    ~~~~
    \includegraphics[width=0.23\textwidth,height=0.16\textheight]{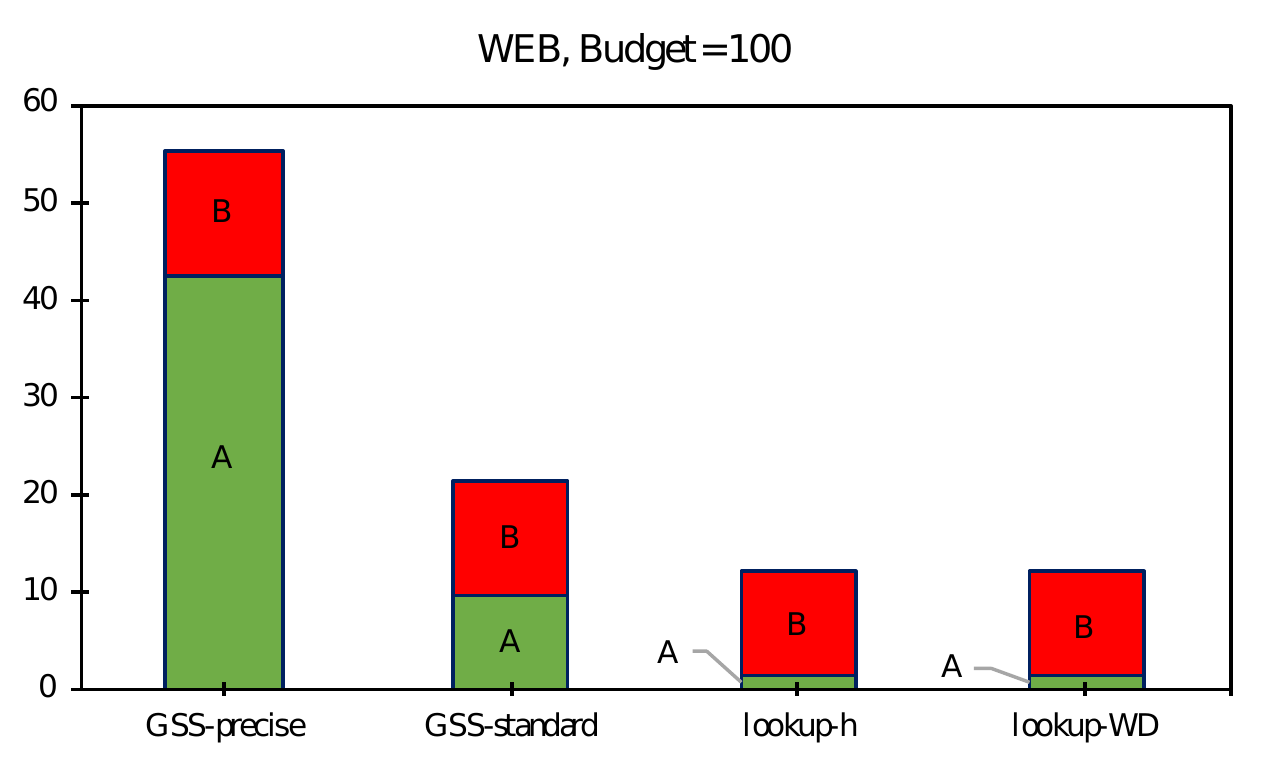}
    \includegraphics[width=0.23\textwidth,height=0.16\textheight]{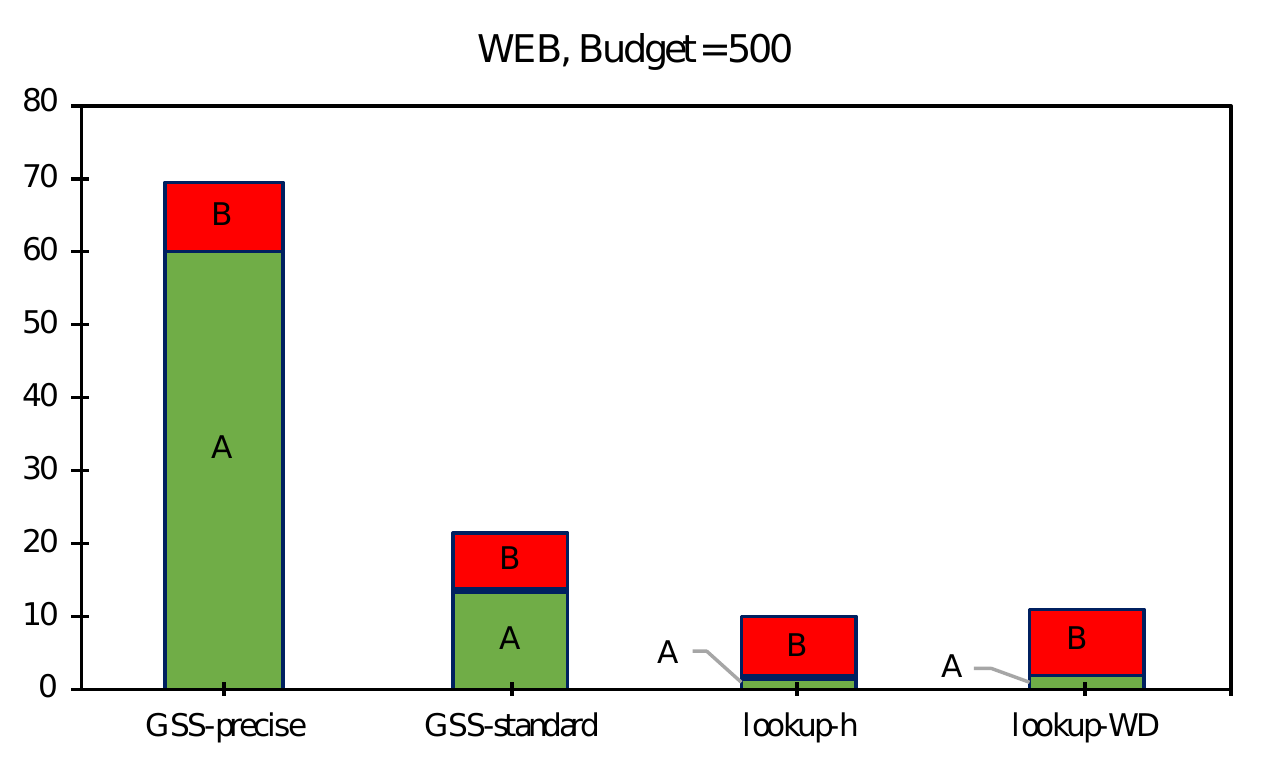}
    \\
    \includegraphics[width=0.23\textwidth,height=0.16\textheight]{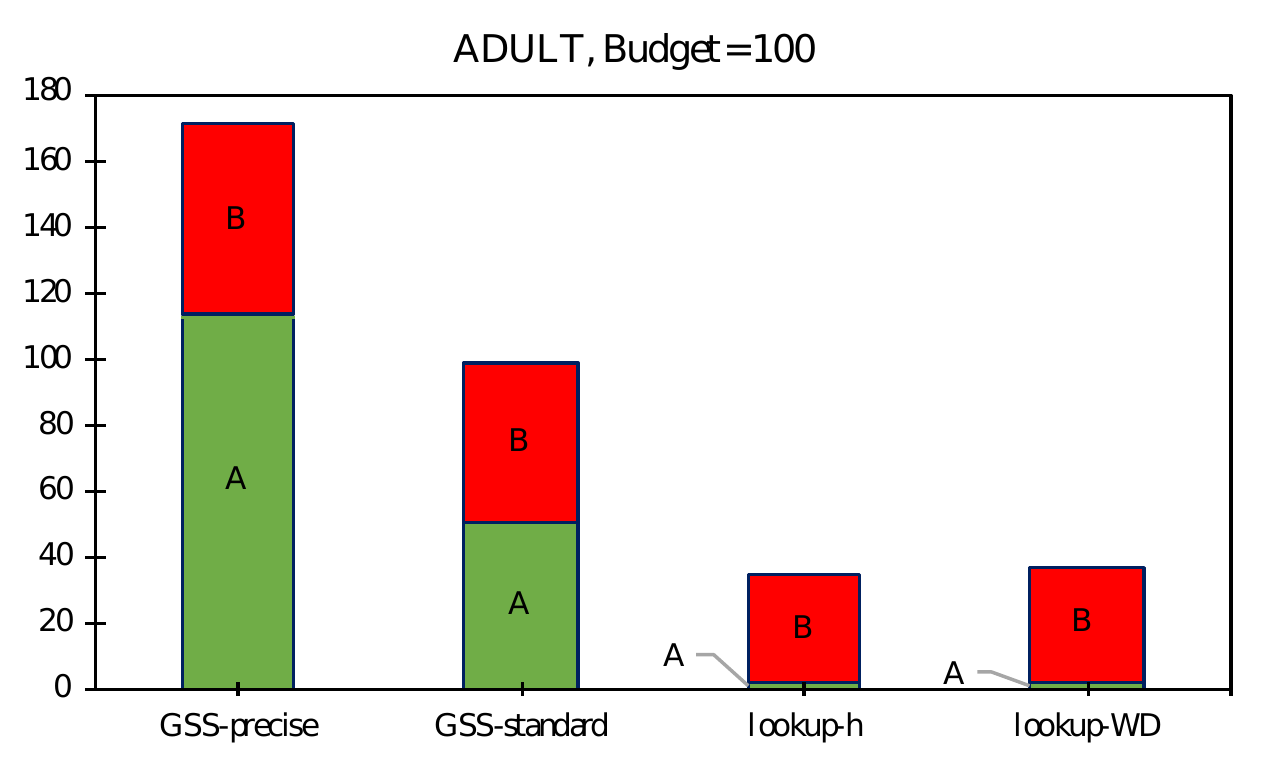}
    \includegraphics[width=0.23\textwidth,height=0.16\textheight]{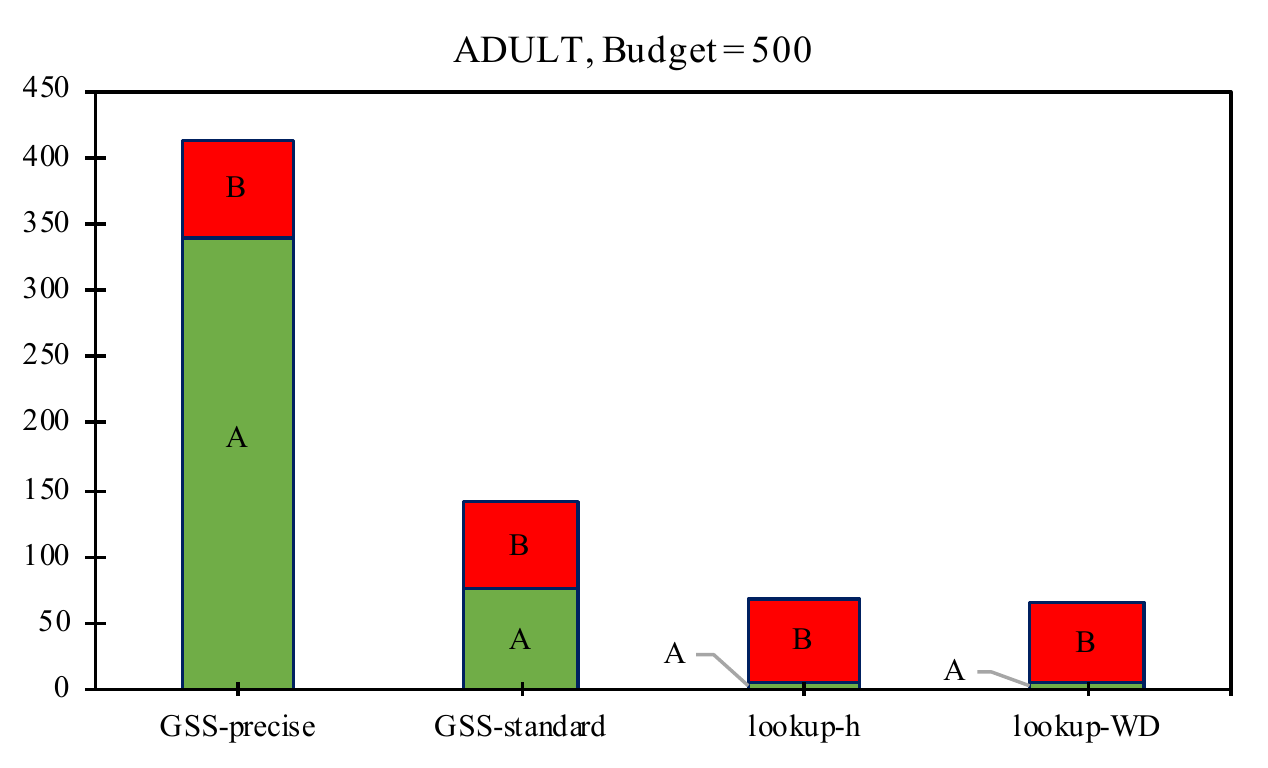}
    ~~~~
    \includegraphics[width=0.23\textwidth,height=0.16\textheight]{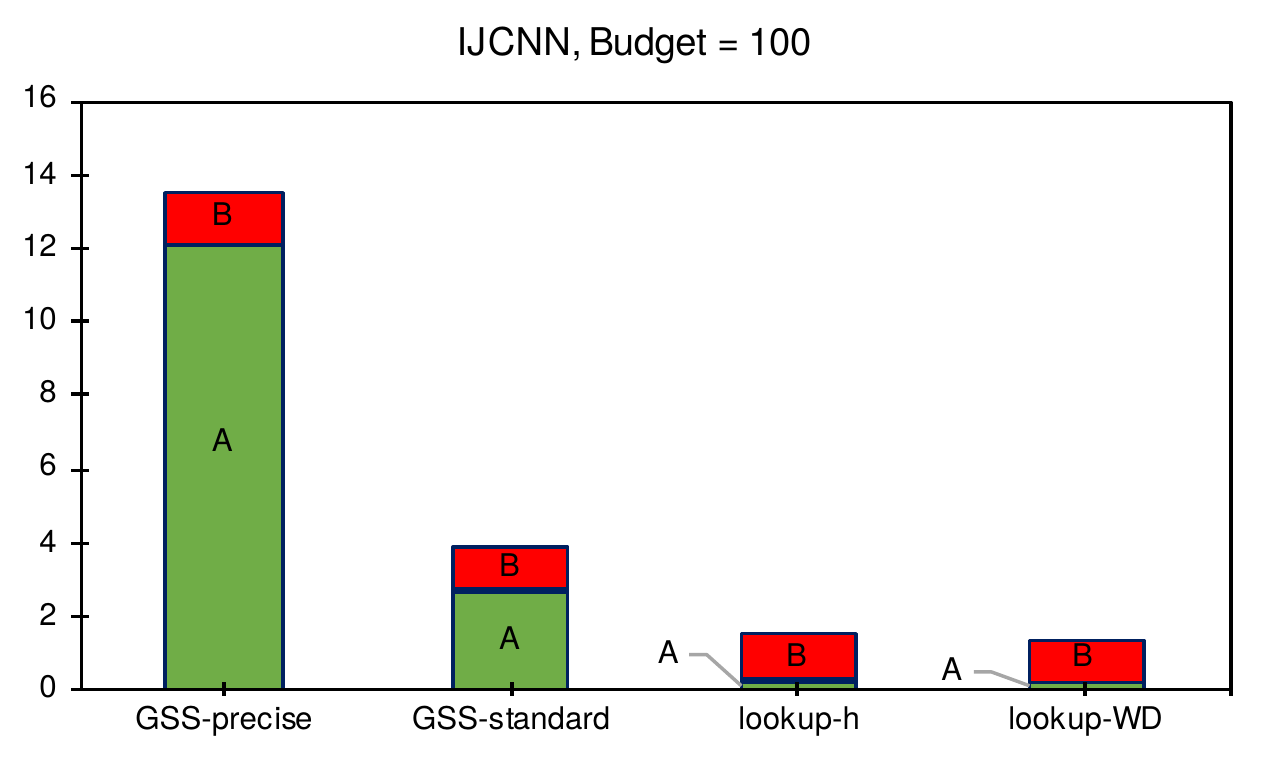}
    \includegraphics[width=0.23\textwidth,height=0.16\textheight]{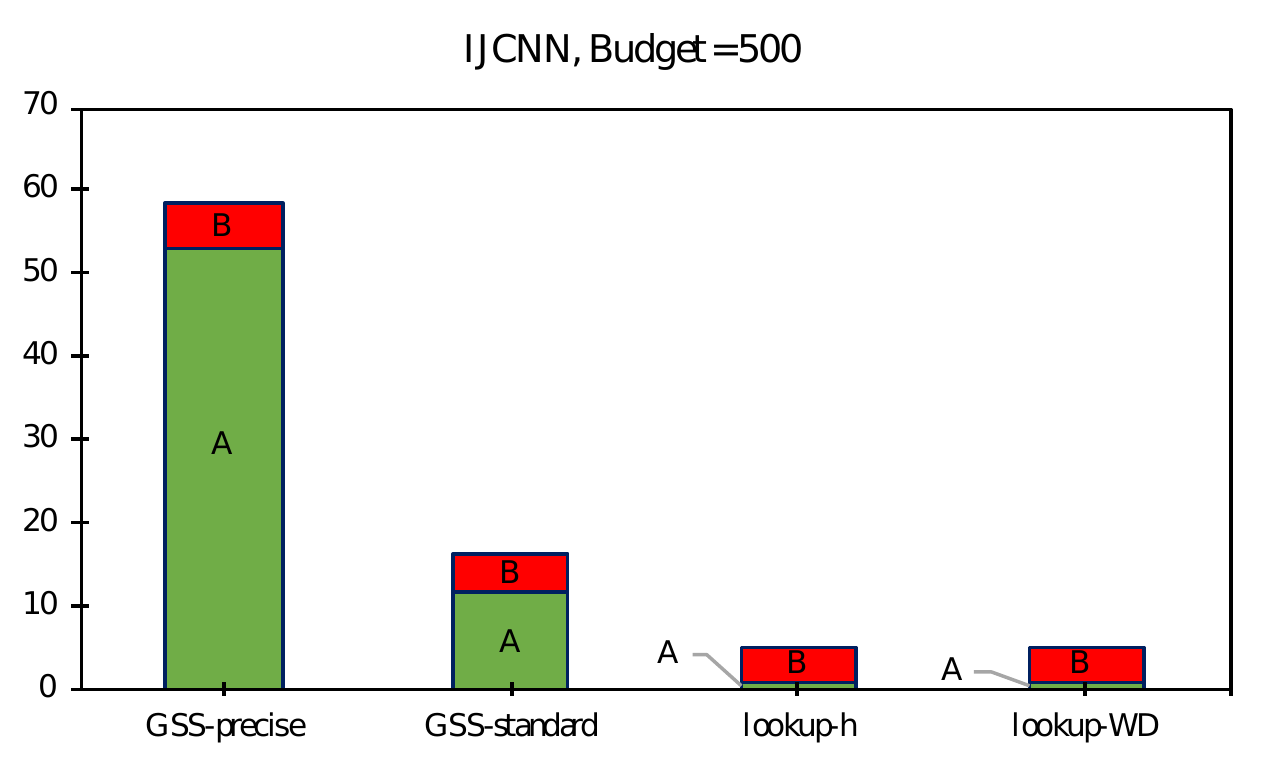}
    \\
    \includegraphics[width=0.23\textwidth,height=0.16\textheight]{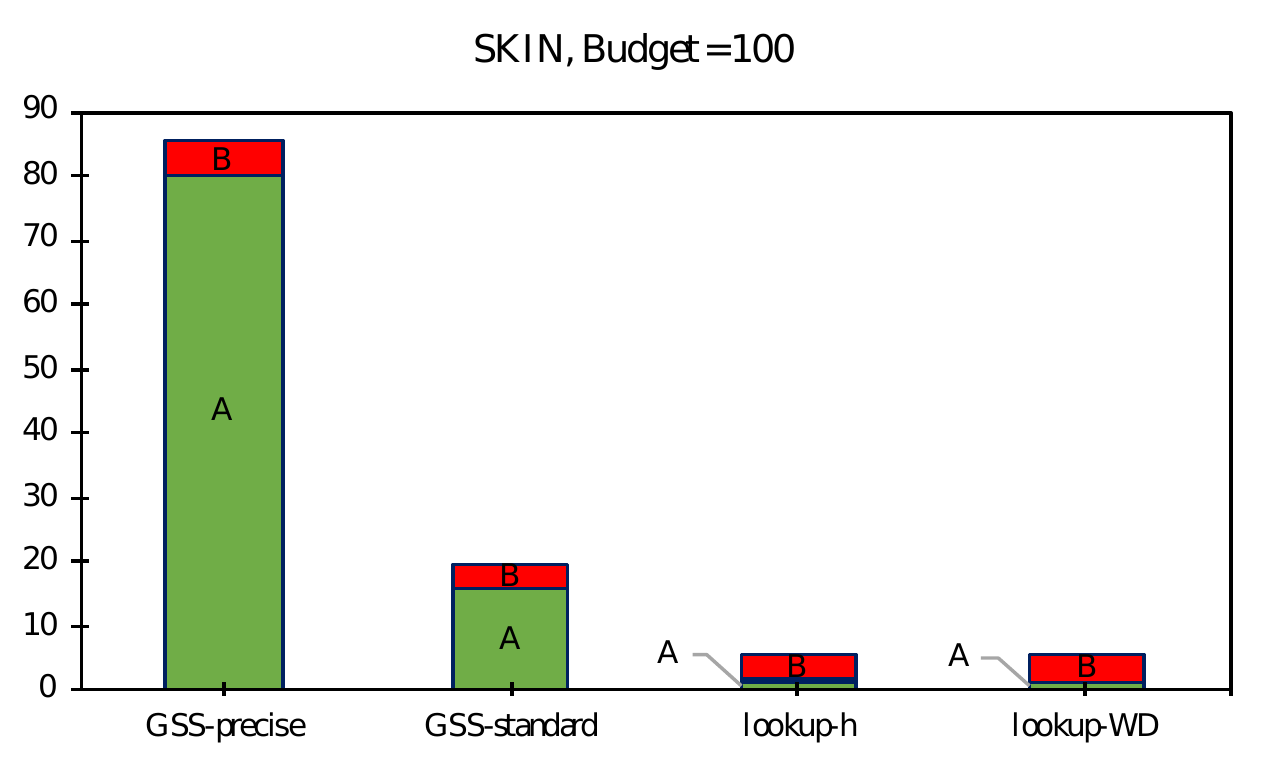}
    \includegraphics[width=0.23\textwidth,height=0.16\textheight]{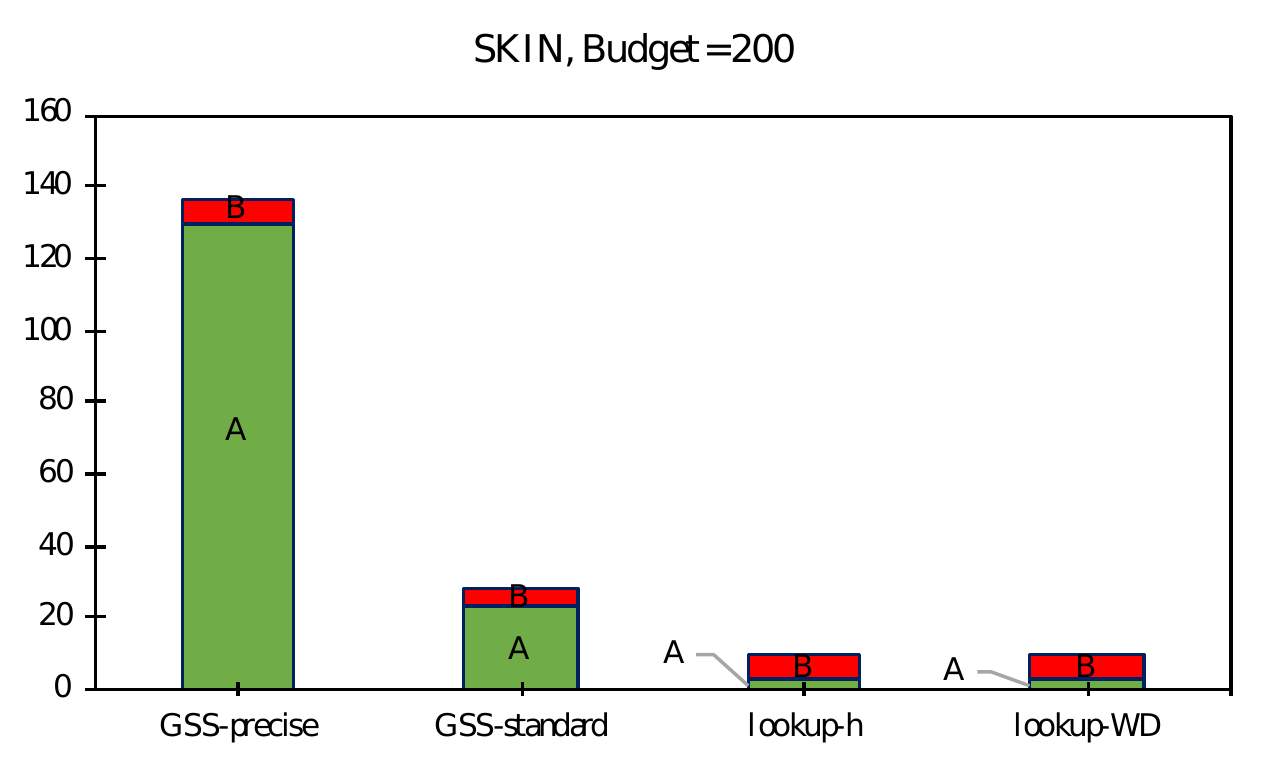}
    ~~~~
    \includegraphics[width=0.23\textwidth,height=0.16\textheight]{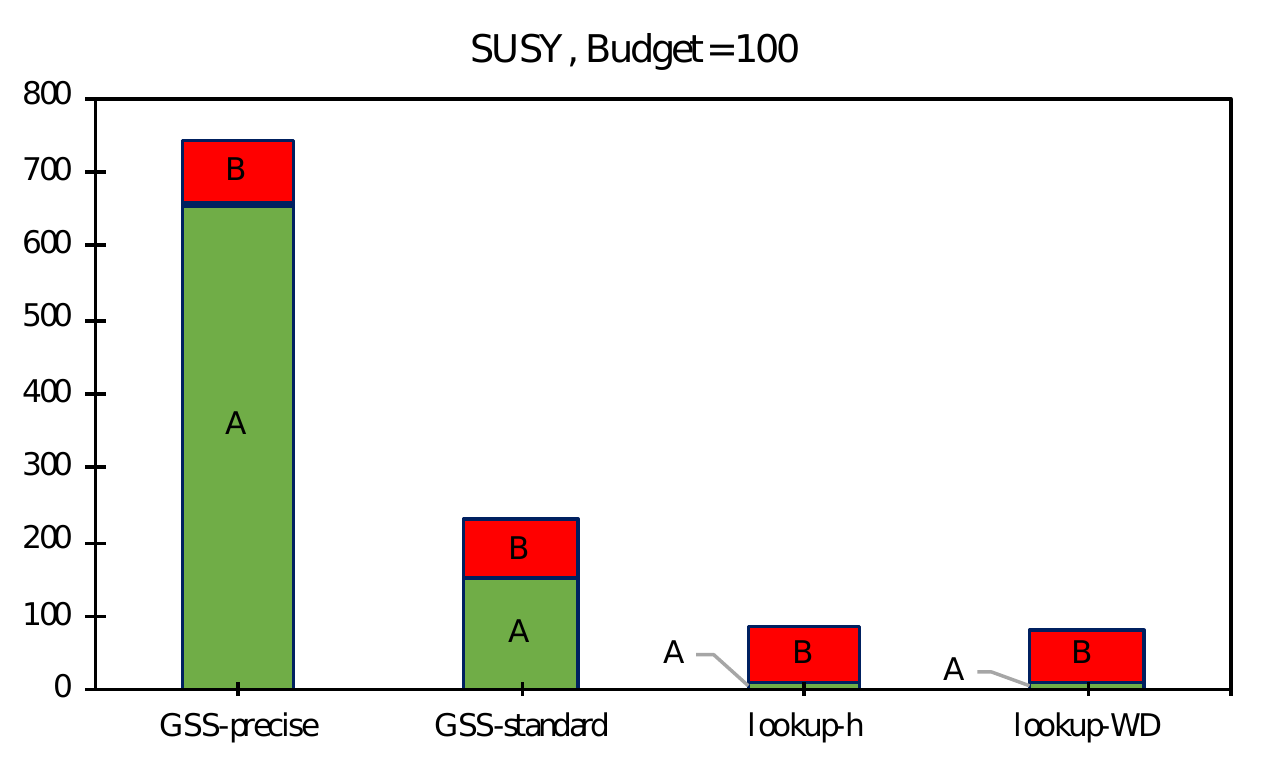}
    \includegraphics[width=0.23\textwidth,height=0.16\textheight]{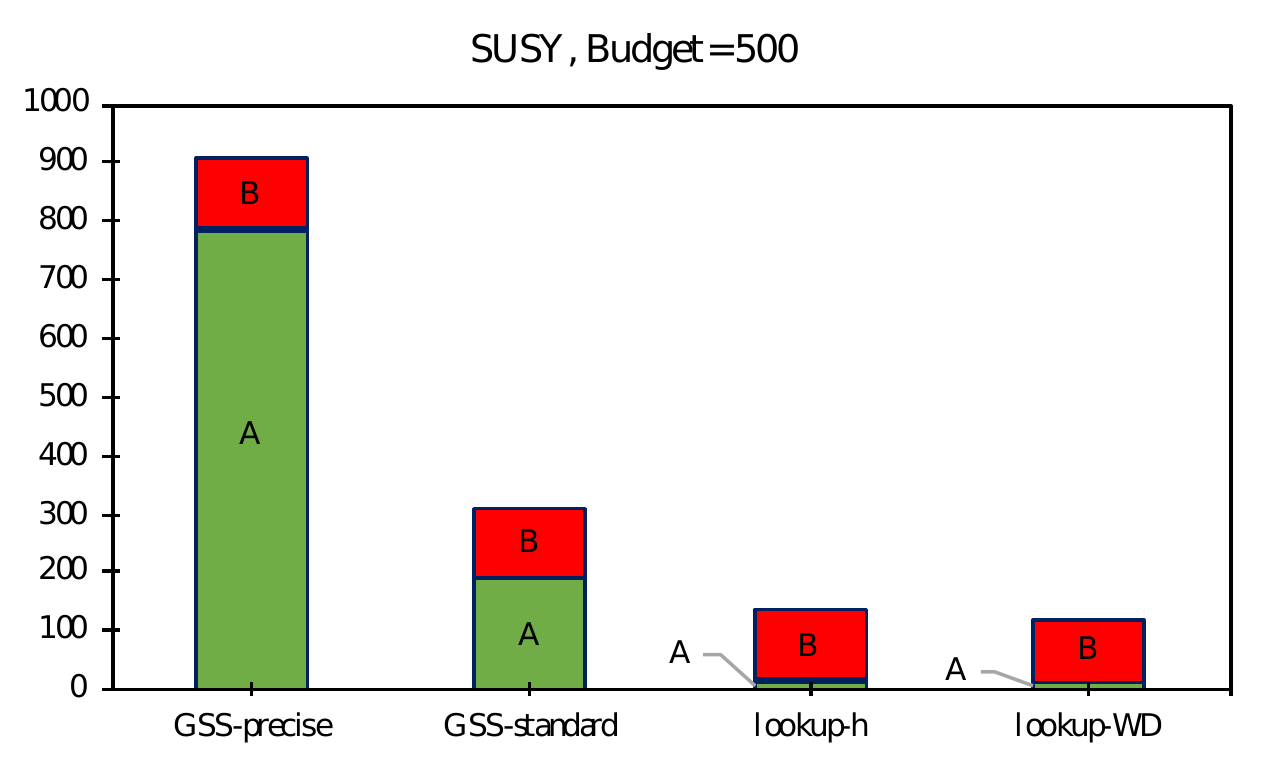}
	\caption{
		\label{figure:TimeProfiles}
		Breakdown of the merging time in seconds for GSS-precise, GSS,
		Lookup-h and Lookup-WD. Section $A$ represents the time invested
		to compute $h$ using either golden section search or lookup. For
		the Lookup-WD method the same bar represents the look-up of
		$W\!D(m,\kappa)$. Section $B$ summarizes all other operations like
		loop overheads, the computation of $\alpha_z$, and the
		construction of the final merge vector $z$.
	}
\end{figure}

\begin{table}
\begin{center}
	\caption{
		\label{table:training-time}
		Relative improvement of the total training time with respect to golden section search
		averaged over 5 runs (Lookup-h vs. GSS-standard and lookup-WD vs. GSS-standard), and
		fraction of merging events for budget size 100 and statistics on the quality of merging
		decisions (refer to the text for details).
	}
	         
	\resizebox{\columnwidth}{!}{
	\begin{tabular}{lrrrrrrr}
	\hline
	data set~ & ~budget~ & ~Lookup-h vs.~ & ~Lookup-WD vs.~ & ~merging~   & ~equal~merging~ & ~factor~    & ~factor~    \\
	          & ~size~   & ~GSS-standard~ & ~GSS-standard~  & ~frequency~ & ~decisions~     & ~GSS~       & ~lookup-WD~ \\
	\hline
	SUSY     & $100$ & $43.911\%$ & $43.396\%$ &  $43\%$ & $93.64\%$     & $1.01795$ & $1.00733$ \\
	         & $500$ & $39.201\%$ & $39.199\%$ \\

	SKIN     & $100$ & $20.515\%$ & $17.788\%$ & $16\%$ & $74.31\%$     & $1.00047$ & $1.00005$ \\
	         & $200$ & $14.173\%$ & $14.900\%$ \\

	IJCNN    & $100$ & $28.091\%$ & $30.372\%$ &   $17\%$ & $91.79\%$     &$1.02429$ & $1.00149$ \\
	         & $500$ & $30.569\%$ & $29.861\%$ \\

	ADULT    & $100$ & $21.627\%$ & $18.452\%$ &  $32\%$ & $92.54\%$     & $1.05064$ & $1.00402$ \\
	         & $500$ & $22.334\%$ & $22.339\%$ \\

	WEB      & $100$ &  $3.053\%$ &  $5.649\%$ &   $6\%$ & $93.77\%$     &$1.00255$ & $1.00039$ \\
	         & $500$ &  $7.483\%$ &  $0.508\%$ \\

	PHISHING & $100$ & $15.385\%$ & $13.946\%$ &  $21\%$ & $96.96\%$     & $1.00055$ & $1.00008$ \\
	         & $500$ &  $7.563\%$ & $10.924\%$ \\
	\hline
	\end{tabular}
	}
\end{center}
\end{table}

\section{Conclusion}
We have proposed a fast lookup as a plug-in replacement for the
iterative golden section search procedure required when merging support
vectors in large-scale kernel SVM training. The new method compares
favorably to the iterative baseline in terms of training time: it offers
a systematic speed-up, resulting in computational savings of up to
$65\%$ of the merging time and up to $44\%$ of the total training time,
while the training time is never increased.
With our method, nearly the full computation time is spent on actual SGD
steps, while the fraction of efforts spent on budget maintenance can be
reduced significantly. We have demonstrated that our approach results in
virtually indistinguishable and even slightly more precise merging
decisions. It is for this reason that the speed-up comes at absolutely no
cost in terms of predictive accuracy.

\section*{Acknowledgments}

We acknowledge support by the Deutsche Forschungsgemeinschaft (DFG)
through grant GL~839/3-1.

\bibliographystyle{plain}

\end{document}